\documentclass[12pt]{article}

\usepackage{amsmath}
\usepackage{amsfonts}
\usepackage{amsthm}
\usepackage{stmaryrd}
\usepackage{graphicx}

\newtheorem{theorem}{Theorem}

\newcommand{\pdfpn}{f_\pm}
\newcommand{\pdfn}{f_-}
\newcommand{\pdfp}{f_+}
\newcommand{\proj}[2] {#1^T#2}
\newcommand{\cipSymbol} {\text{ip}^\times}
\newcommand{\cip}[1] {\cipSymbol(#1)}
\newcommand{\crossEntropySymbol}{\text{H}_2^\times}
\newcommand{\crossEntropy}[1] {\crossEntropySymbol(#1)}
\newcommand{\entropy}[1] {\text{H}_2(#1)}
\newcommand{\dcsSymbol}{\text{D$_{\text{CS}}$}}
\newcommand{\dcs}[1] {\dcsSymbol(#1)}

\usepackage{booktabs}
\begin{document}

\author{Rafal Jozefowicz$^1$ and Wojciech Marian Czarnecki$^2$}

\date{
$^1$Google Inc.\\ e-mail: {\it rafjoz@gmail.com}\\
$^2$Faculty of Mathematics and Computer Science\\ Jagiellonian University\\ ul. Lojasiewicza 6, 30-389 Krakow\\ e-mail: {\it w.czarnecki@uj.edu.pl}
}

\title{Fast optimization of \\Multithreshold Entropy Linear Classifier}

\maketitle

\abstract{Multithreshold Entropy Linear Classifier (MELC) is a density based model which searches for a linear projection maximizing the Cauchy-Schwarz Divergence of dataset kernel density estimation. Despite its good empirical results, one of its drawbacks is the optimization speed. In this paper we analyze how one can speed it up through solving an approximate problem. We analyze two methods, both similar to the approximate solutions of the Kernel Density Estimation querying and provide adaptive schemes for selecting a crucial parameters based on user-specified acceptable error. Furthermore we show how one can exploit well known conjugate gradients and L-BFGS optimizers despite the fact that the original optimization problem should be solved on the sphere. All above methods and modifications are tested on 10 real life datasets from UCI repository to confirm their practical usability.}
% In particular we show a simple algorithm which reduces iteration complexity from $O(N^2)$ to nearly $O(N)$.}

% \keywords{multithreshold classifier, entropy, approximation, optimization} 

\section{Introduction}

Many methods of speeding up the kernel density estimator's (KDE) querying process has been proposed in the literature~\cite{silverman1982algorithm,yang2003improved,elgammal2003efficient}. As optimization problem introduced in Multithreshold Entropy Linear Classifier~\cite{melc} is closely related to the equations of KDE it appears natural that similar techniques can be used to simplify its computations with a bounded error. Importance of such reductions comes from the high (quadratic) complexity of the evaluation of functions required during training of this model which makes it hard to use for any dataset with more than a thousand points. In this paper we investigate two such approaches, first -- sorting and discarding, which ignores computations of similarities between points that are too far away to have big impact on the function's value, second -- binning, which smooths the function construction in order to heavily reduce amount of unique points. Both these methods are introduced in an adaptive manner 
so the 
optimization process have fixed error bound despite many 
different linear projections being analyzed during the training phase. We also show a very simple method which enables to use a wide range of optimization algorithms even though proposed model requires optimization with a specific constraints (sphere bounded).

\section{Multithreshold Entropy Linear Classifier}
Multithreshold Entropy Linear Classifier (MELC~\cite{melc}) has been recently proposed as an information theoretic approach for building model from the multithreshold linear family~\cite{anthony}. It's core idea is to find a linear operator $v$ (with unit norm) such that kernel density estimations of projected classes' training samples maximize the Cauchy-Schwarz Divergence ($\dcsSymbol$~\cite{principe2010information}). Let us recall the equation of $\dcsSymbol$ in order to find the core computational bottleneck which appears in MELC optimization 
$$
\dcs{\pdfn, \pdfp} = 2\crossEntropy{\pdfn, \pdfp} - \entropy{\pdfn} - \entropy{\pdfp}, 
$$
for $\pdfpn = \llbracket\proj{v}{X_\pm}\rrbracket$ being a kernel density estimator of $\proj{v}{X_\pm}$ with Silverman's rule~\cite{silverman}, thus from the definition of Renyi's quadratic entropy, Renyi's quadratic cross entropy and the fact that $\cip{f,g} = \int fg $ we have
$$
\dcs{\pdfn, \pdfp} = - 2 \log \cip{\pdfn, \pdfp} + \log \cip{\pdfn, \pdfn} + \log \cip{\pdfp,\pdfp}.
$$
As whole $\dcsSymbol$ function is composed of $\cipSymbol$ evaluations, in the rest of our paper we focus purely on the $\cipSymbol$, which we expand using Gaussian kernel density estimation~\cite{melc} and denote $\cip{v}=\cip{\llbracket\proj{v}{X_-}\rrbracket, \llbracket\proj{v}{X_+}\rrbracket}$.
$$
\cip{v} = \frac{1}{\sqrt{2 \pi V(v) |X||Y|}} \sum_{x,y} \exp \left ( -\frac{\langle v, x-y \rangle^2}{2V(v)} \right ),
$$
where $V(v)$ is a sum of each classes estimated variances using Silverman's rule~\cite{silverman}.

In an obvious way, naive computation of $\cipSymbol$ is $\mathcal{O}(N^2)$, where $N=\max\{|X_-|,|X_+|\}$ due to the summation over all possible pairs $(x,y)$. In the following sections we focus on methods which reduce this computational bottleneck while still preserving given approximation of $\cipSymbol$ value.

\section{Reduction of $\cipSymbol$ computational complexity}

\paragraph{Sorting and discarding}
Let us begin with the very simple conception of computing values of only those $(x,y)$ pairs which are close enough to have an impact on the value of $\cipSymbol$. If we assume that points projections are sorted (which can be done in general in $\mathcal{O}(N \log N)$\footnote{in fact for iterative optimization techniques points ordering does not change much between subsequent calls so after initial sorting it can be done in linear time using insertion sort}) we can search the dataset in linear time and identify for each point $x$ indices of first and last point which are at most at distance $T$ from $x$. Following theorem shows what $T$ to choose in order to obtain at most $\epsilon$ error.

\begin{theorem}
 Using adaptive sorting and discarding with distance threshold in each iteration of at least $$\sqrt{ \max \left \{ 0 , - V(v) \ln \left (  2 (\tfrac{\epsilon}{p})^2 \pi V(v) \right ) \right \} }, $$ where $V(v)$ is a sum of each classes estimated variances, leads to the computation of the $\cipSymbol$ function with at most $\epsilon$ error, assuming that at most fraction of $p$ points is located closer than $T$.
\end{theorem}
\begin{proof}
We assume that $|\langle v,  x-y \rangle| \geq T$ for $N_T$ pairs of points which are being ignored during computation of $\cipSymbol$ so
$ - \langle v, x-y \rangle^2 \leq - T^2, $
thus
$$
\frac{1}{\sqrt{2 \pi V(v)} |X||Y|} \sum_{x,y} \exp \left ( -\frac{\langle v, x-y \rangle^2}{2V(v)} \right )
\leq
$$
$$
\frac{1}{\sqrt{2 \pi V(v)} |X||Y|} \sum_{x,y} \exp \left ( -\frac{T^2}{2V(v)} \right )
=
\frac{N_T}{\sqrt{2 \pi V(v) } |X||Y|}  \exp \left ( -\frac{T^2}{2V(v)} \right ).
$$
If we look for an $\epsilon$ approximation of non-regularized MELC objective
we put $0 \leq p = N_T/(|X||Y|) \leq 1$ and consequently
$$
 p\frac{1}{\sqrt{2 \pi V(v) }}   \exp \left ( -\frac{T^2}{2V(v)} \right ) \leq \epsilon,
$$
thus
$$
\exp \left ( -\frac{T^2}{2V(v)} \right ) \leq \tfrac{\epsilon}{p} \sqrt{ 2 \pi V(v) }
$$

$$
T^2 \geq -2V(v) \ln \left ( \tfrac{\epsilon}{p} \sqrt{2 \pi V(v)}  \right ),
$$
obviously if $\ln \left ( \tfrac{\epsilon}{p} \sqrt{2 \pi V(v)}  \right ) > 0$ then any $T$ satisfies this inequality (as it can only happen if we choose very big acceptable error $\epsilon$), so for simplicity we add the maximum of this value with $0$.
$$
T \geq \sqrt{ \max \left \{ 0, - 2V(v) \ln \left ( \tfrac{\epsilon}{p} \sqrt{2 \pi V(v)} \right ) \right \} } = \sqrt{ \max \left \{ 0, - V(v) \ln \left (  2 (\tfrac{\epsilon}{p})^2 \pi V(v) \right ) \right \}  }. 
$$

\end{proof}
% 
% $$
% T \geq \sqrt{ - 2V(v) ( (\epsilon/2) + \ln \left ( \sqrt{2 \pi V(v)} \right ) )}
% $$
% 
% $$
% T \geq \sqrt{ - V(v)( \epsilon + \ln \left ( \frac{2 \pi V(v) }{|X||Y|} \right ) )}
% $$
% 
% $$
% T \geq \sqrt{ V(v)( \ln \left ( \frac{|X||Y|}{2 \pi V(v) } \right ) - \epsilon)}
% $$

\paragraph{Binning}

While sorting and discarding technique is quite easy to implement and analyze its practical speedup might be limited for densely packed datasets. In such cases it might be more valuable to perform a binning of our projected points, so those located near each other are approximated by their empirical mean. Such an approach works well for densely packed datasets which makes it a complementary approach to the previous one.

Let us assume that we have some partitioning of the $\mathbb{R} = \bigcup_{i=1}^k a_i$ where each $a_i$ is an interval. We define a binning operator as $b(x) = \text{mean} \{x \in \proj{v}{X} \cap a_{i(x)} \}$, where $x \in a_{i(x)}$. We use following notation for simplicity $\langle v, x \rangle_b = b(\langle v,x \rangle)$ .Similarly to the previous strategy, in order to preserve good approximation, bins width ($B=\max_i | a_i |$) needs to be adapted in each iteration and the exact equation is given in the following theorem.

\begin{theorem}
 Using adaptive binning technique with bin width in each iteration at most $$\sqrt{ - 2V(v) \ln \left ( \max \left \{ 0,1-  \epsilon \sqrt{2 \pi V(v)} \right \} \right ) }$$ where $V(v)$ is a sum of each classes estimated variances, leads to the computation of the $ip^\times$ function with at most $\epsilon$ error.
\end{theorem}

\begin{proof}
we assume that $|\langle v,  x-y \rangle - (\langle v,  x \rangle_b - \langle v, y \rangle_b) | \leq B$ so
$$
\left | \cip{v}
-
\frac{1}{\sqrt{2 \pi V(v)} |X||Y|} \sum_{x,y} \exp \left ( -\frac{(\langle v,  x \rangle_b - \langle v, y \rangle_b)^2}{2V(v)} \right ) \right |
=
$$

$$
\left  | \frac{1}{\sqrt{2 \pi V(v)} |X||Y|} \sum_{x,y} \left [ \exp \left ( -\frac{\langle v, x-y \rangle^2}{2V(v)} \right )
-
 \exp \left ( -\frac{(\langle v,  x \rangle_b - \langle v, y \rangle_b)^2}{2V(v)} \right ) \right ] \right |
\leq
$$

$$
\left  | \frac{1}{\sqrt{2 \pi V(v)} |X||Y|} \sum_{x,y} \left [ \exp \left ( 0 \right )
-
 \exp \left ( -\frac{B^2}{2V(v)} \right ) \right ] \right |
=
$$

$$
\left | {\frac{1}{\sqrt2 \pi V(v) }}  \left [ 1
-
 \exp \left ( -\frac{B^2}{2V(v)} \right ) \right ] \right |.
$$
Let us now assume that we are given some acceptable error $\epsilon \geq 0$. We will show how small bins have to be used based on our dataset and current projection.
$$
\left | {\frac{1}{\sqrt2 \pi V(v) }}  \left [ 1
-
 \exp \left ( -\frac{B^2}{2V(v)} \right ) \right ] \right | \leq \epsilon,
$$
but $\exp \left ( -\frac{B^2}{2V(v)} \right ) \leq 1$, so
$$
{\frac{1}{\sqrt2 \pi V(v) }}  \left [ 1
-
 \exp \left ( -\frac{B^2}{2V(v)} \right ) \right ]  \leq \epsilon,
$$
thus
$$
 \exp \left ( -\frac{B^2}{2V(v)} \right )   \geq 1 - \epsilon \sqrt{2 \pi V(v)}.
$$
Naturally if $1 - \epsilon \sqrt{2 \pi V(v)} < 0$ then any $B$ satisfies this inequality (similarly to the sorting and discarding method, it may only happen if we choose very large acceptable error $\epsilon$) so we introduce maximum function here.

$$
-\frac{B^2}{2V(v)}   \geq \ln \left ( \max \left \{ 0, 1 -  \epsilon \sqrt{2 \pi V(v)}\right \} \right ) 
$$

$$
B \leq \sqrt{ - 2V(v) \ln \left ( \max \left \{ 0,1-  \epsilon \sqrt{2 \pi V(v)} \right \} \right ) }
$$
\end{proof}

Figure~\ref{fig:bounds} shows how these two bounds behave with increasing size of the acceptable error. In particular one can see that both methods have very similar growth (up to the maximization/minimization symmetry) with changing $\epsilon$. As a result, due to the fact that binning is much more aggressive technique we should expect that using these bounds as the actual bin width/discarding threshold will lead to much greater reduction of the computational complexity when using binning.
\begin{figure}[htb]
\includegraphics[width=0.5\textwidth]{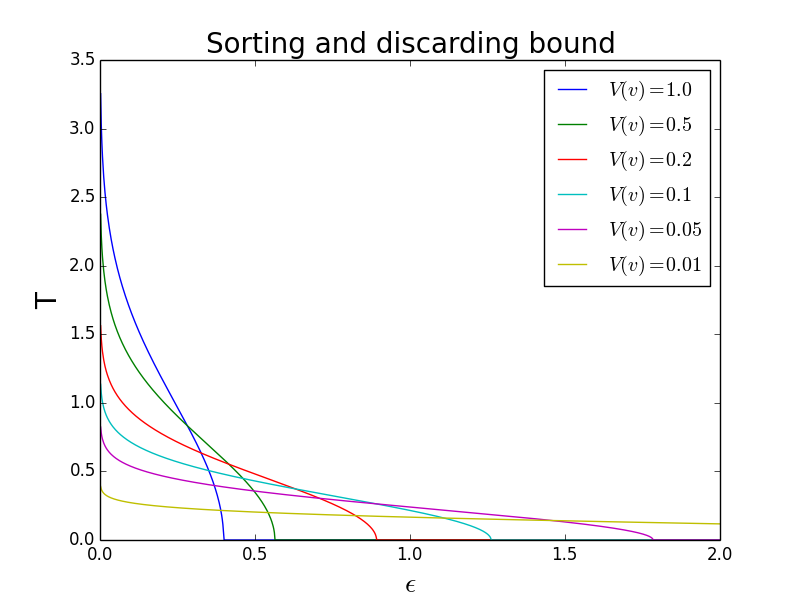}
 \includegraphics[width=0.5\textwidth]{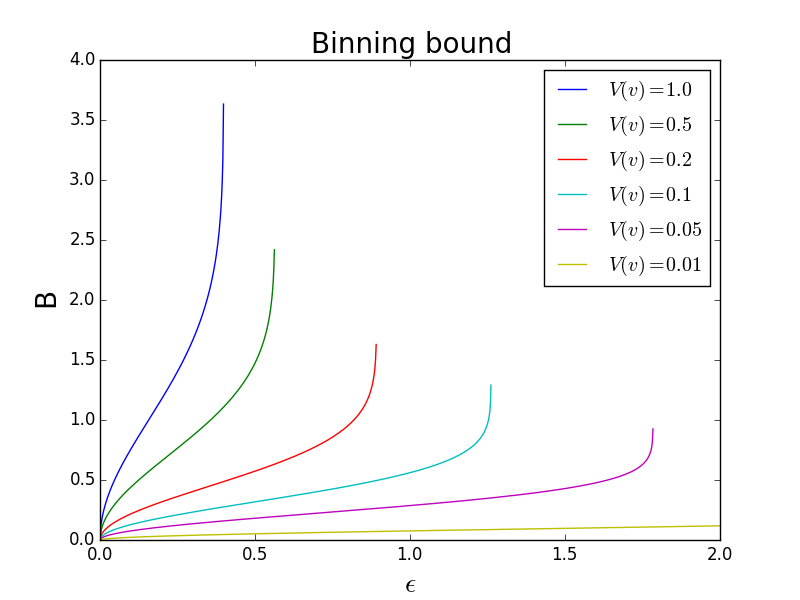}
 \caption{Plots of the values of the discarding threshold (on the left) and bin width (on the right) as the function of the acceptable error $\epsilon$.}
 \label{fig:bounds}
\end{figure}

\section{Out of sphere optimization}

Now we are going to show, that MELC objective function can be efficiently optimized in the whole $R^d$ space by adding some custom regularization term. The importance of this result is the fact that it enables us to use vast amount of existing optimization techniques (such as Adaptive gradient descent, Conjugate Gradients, BFGS, L-BFGS etc.) without adapting them to the sphere constraints. The second important aspect is the fact that this modification does not involve adding any additional constants which have to be fitted. Following theorem describes modified objective function.

\begin{theorem}
 Given arbitrary sets $X_-, X_+ \subset \mathbb{R}^d$ and corresponding $\dcsSymbol(v) = \dcsSymbol(\llbracket \proj{v}{X_-}\rrbracket, \llbracket \proj{v}{X_+} \rrbracket)$ function we have:
 $$
 d := \max_{\| v \| =1} \dcsSymbol(v) =  \max_v \dcsSymbol(v) - (\|v\|^2 - 1)^2 
 $$
 and
 $$
 \{ v : \|v\|=1 \wedge  \dcsSymbol(v) = d\}  = \{ v :  \dcsSymbol(v)  - (\|v\|^2 - 1)^2 = d\}.
 $$
%  we have
%  $$
%  \frac{v_{\dcsSymbol}}{\|v_{\dcsSymbol}\|} = \frac{v_{D'_{CS}}}{\|v_{D'_{CS}}\|}
%  $$
\end{theorem}
\begin{proof}

 According to~\cite{melc}, $\dcsSymbol$ is scale invariant so for any $v \in \mathbb{R}^d, c \in \mathbb{R}_+$
 $$
 \dcsSymbol(v) = \dcsSymbol(cv).
 $$
  As a result also
 $$
 \dcsSymbol(v) - (\|v\|^2 - 1)^2
 =
 \dcsSymbol(cv) - (\|v\|^2 - 1)^2
 $$ 
 but as $- (\|v\|^2 - 1)^2 \leq 0$ and $- (\|v\|^2 - 1)^2 = 0 \iff \|v\|=1$ we have that 
 $\dcsSymbol(v) - (\|v\|^2 - 1)^2$ is maximized for $v$ with norm $1$ and that it is equal to
 $\dcsSymbol(v)$. As a result sets of solutions of both problems are identical.

%  $$
%  \dcsSymbol(\llbracketv^TX_-\rrbracket, \llbracketv^TX_+\rrbracket) = \dcsSymbol(\llbracketv'^TX_-\rrbracket, \llbracketv'^TX_+\rrbracket)
%  $$
%  but obviously $\|v'\|=1$ thus
%  $$
%  \dcsSymbol(\llbracketv'^TX_-\rrbracket, \llbracketv'^TX_+\rrbracket) = \dcsSymbol(\llbracketv^TX_-\rrbracket, \llbracketv^TX_+\rrbracket) - (\|v\|^2 - 1)^2
%  $$
%  so 
%  $$
%  \max_v \dcsSymbol(\llbracketv^TX_-\rrbracket, \llbracketv^TX_+\rrbracket) \leq \max_v \dcsSymbol(\llbracketv^TX_-\rrbracket, \llbracketv^TX_+\rrbracket) - (\|v\|^2 - 1)^2
%  $$
%  Inequality in opposite direction is a natural consequence of the fact that $- (\|v\|^2 - 1)^2 \leq 0$ for any $v \in \mathbb{R}^d$ which completes the proof.
\end{proof}

Consequently we can apply any advanced optimization technique which is not designed to work on the sphere to optimize $\dcsSymbol$ criterion. In particular we can use L-BFGS~\cite{byrd1995limited} instead of more complex and less popular RBFGS~\cite{qi2010riemannian} and previously proposed~\cite{melc} less efficient -- gradient descent on sphere method. At the same time the norm of the candidate solution will stay close to $1$ so we will not suffer from numerical problems~\cite{melc}.

It is worth noting that despite similarity to the L$_2$ regularization~\cite{vapnik2000nature} of the additive loss function (or weight decay from neural networks) this additional terms serves no regularization purposes nor it affects the actual function value. It only guides the gradient based optimizers towards more informative regions of the state space.

From the practical point of view we also need a gradient of the new function but thanks to the additivity of derivative operator we get
$$
\nabla \left [ \dcsSymbol(v) - (\| v \|^2-1)^2 \right ] = \left [ \nabla \dcsSymbol(v) \right ] - 4 v ( \langle v, v \rangle - 1 ),
$$
and we can use any optimization software able to maximize a function given ($f, \nabla f)$.
% runing it with $(\dcsSymbol(v)  - (\|v\|^2 - 1)^2, \left [ \nabla \dcsSymbol(v) \right ] - 4 v ( \langle v, v \rangle - 1 ))$.
% where $a \cdot b$ is an element wise vector multiplication.
% 
% let as assume that $\langle v, x-y \rangle < T$, so $0 \geq -\langle v, x-y \rangle^2 > - T^2$ thus
% 
% $$
% -\frac{T^2}{2V(v)} < -\frac{\langle v, x-y \rangle^2}{2V(v)} \leq 0
% $$ 
% 
% and
% 
% $$
% -\frac{(T+B)^2}{2V(v)} < -\frac{(b(\langle v,  x \rangle ) - b(\langle v, y \rangle))^2}{2V(v)} \leq 0
% $$ 

% 
% $$
% \exp(-\frac{T^2}{2V(v)}) < \exp(-\frac{\langle v, x-y \rangle^2}{2V(v)}) \leq 1
% $$ 

% $$
% \frac{1}{\sqrt{2 \pi V(v) |X||Y|}} \sum_{x,y}  \exp \left ( -\frac{\langle v, x-y \rangle^2}{b(\langle v, x-y \rangle)^2} \right )
% =
% $$
% $$ - \langle v, x-y \rangle^2 \leq - T^2 $$

% thus
% 
% $$
% \frac{1}{\sqrt{2 \pi V(v) |X||Y|}} \sum_{x,y} \exp \left ( -\frac{\langle v, x-y \rangle^2}{2V(v)} \right )
% \leq
% $$
% $$
% \frac{1}{\sqrt{2 \pi V(v) |X||Y|}} \sum_{x,y} \exp \left ( -\frac{T^2}{2V(v)} \right )
% =
% \sqrt{\frac{|X||Y|}{2 \pi V(v) }}  \exp \left ( -\frac{T^2}{2V(v)} \right )
% $$

\section{Evaluation}

We evaluate proposed approximations on 10 datasets from UCI repository~\cite{uci} and libSVM's repository~\cite{chang2011libsvm, ho1996building}. Both $\dcsSymbol$ and approximations are coded in Python using numpy and scipy~\cite{jones2001scipy}. We use scipy's optimization module to perform training of all models using two optimization techniques -- Conjugate Gradients (CG) and L-BFGS-B~\cite{byrd1995limited}. Each experiment is performed in cross validation manner with multiple starting points (randomly selected, but constant across methods to achieve comparable results) due to the convergence of MELC optimization to local optima. We analyze $\gamma$ hyperparameter of $\dcsSymbol$ in $[0.1, 0.5, 1.0, 1.5, 2.0]$ and acceptable error $\epsilon \in [0.01, 0.02, 0.03, 0.05, 0.1, 0.2, 0.5]$. Similarly to the original paper we use Balanced Accuracy (BAC\footnote{$\text{BAC}=\tfrac{1}{2}\left ( \tfrac{\text{TP}}{\text{TP}+\text{FN}} + \tfrac{\text{TN}}{\text{TN}+\text{FP}} \right )$}) as the measure of 
classification 
correctness due to MELC highly balanced formulation.

First, we investigate how big is mean reduction of computations using each of the approximating schemes. Table~\ref{tab:ratio} reports mean ratio of exp function calls (which is equivalent to number of pairs analyzed in each $\cipSymbol$ evaluation when optimizing whole $\dcsSymbol$ function and its gradient) in given method to the original implementation.
\begin{table}[]
\begin{center}
\begin{tabular}{lrrrr}
\toprule
method 		       &   \multicolumn{2}{c}{CG}       & \multicolumn{2}{c}{L-BFGS-B}          \\
name 		       &   bin &  dist &       bin &  dist \\
\midrule
australian      &  0.11 &  0.44 &      0.11 &  0.45 \\
breast-cancer   &  0.10 &  0.46 &      0.10 &  0.46 \\
diabetes        &  0.21 &  0.56 &      0.22 &  0.54 \\
fourclass       &  0.19 &  0.51 &      0.19 &  0.49 \\
german.numer    &  0.15 &  0.47 &      0.19 &  0.46 \\
heart           &  0.29 &  0.47 &      0.26 &  0.47 \\
ionosphere      &  0.25 &  0.55 &      0.24 &  0.54 \\
liver-disorders &  0.29 &  0.65 &      0.31 &  0.67 \\
sonar           &  0.32 &  0.53 &      0.29 &  0.50 \\
splice          &  0.19 &  0.44 &      0.16 &  0.43 \\
% toy                   &  0.16 &  0.33 &      0.15 &  0.32 \\
\bottomrule
\end{tabular}
\caption{Mean ratio of $\exp$ calls between approximated technique and original method during optimizations.}
\label{tab:ratio}
\end{center}
\end{table}

One can easily notice that sorting and discarding method (denoted as "dist") roughly halves the number of analyzed pairs, while binning (denoted as "bin") reduces it 3-10 times. It is an obvious consequence of the fact that binning is much more aggressive method. It appears that strength of reduction depends only on the dataset, not on the optimization algorithm used which suggests, that projections for which particular level of possible reduction  are uniformly distributed over the space of all projections. These effects are also heavily dependent\footnote{we do not include the exact values in the Table for better readability} on the choice of $\gamma$ and $\epsilon$ which is the obvious consequence of  Theorems 1 and 2 saying that with increasing variance (which is proportional to $\gamma^2$) the reduction strength decreases superlinearly.

The set of heat maps in Figure~\ref{fig:comp} shows differences between BAC obtained by the original $\dcsSymbol$ and each approximation for a given dataset and $\gamma, \epsilon$ hyperparameters pair. In general, up to few isolated cases errors are on the level of $0.5\%-3\%$. For small $\gamma$ values errors introduced by the approximation are significantly higher and for sonar and splice datasets can grow to even $10\%$. Fortunately, these are very rare phenomena.
\begin{figure}[htb]
\begin{center}
 \includegraphics[width=0.19\textwidth]{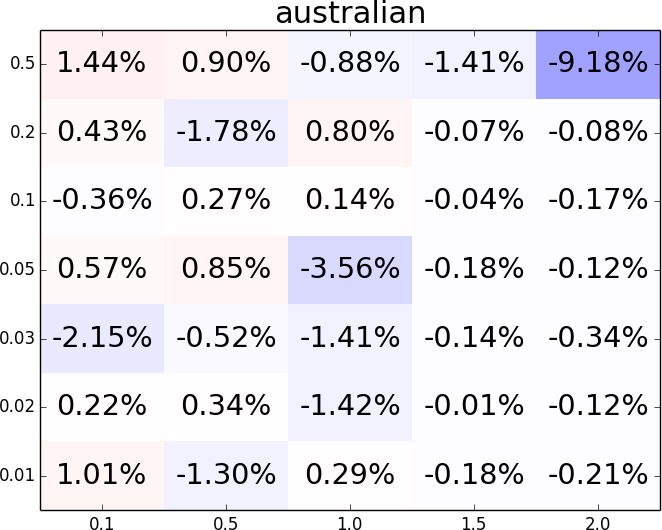}
 \includegraphics[width=0.19\textwidth]{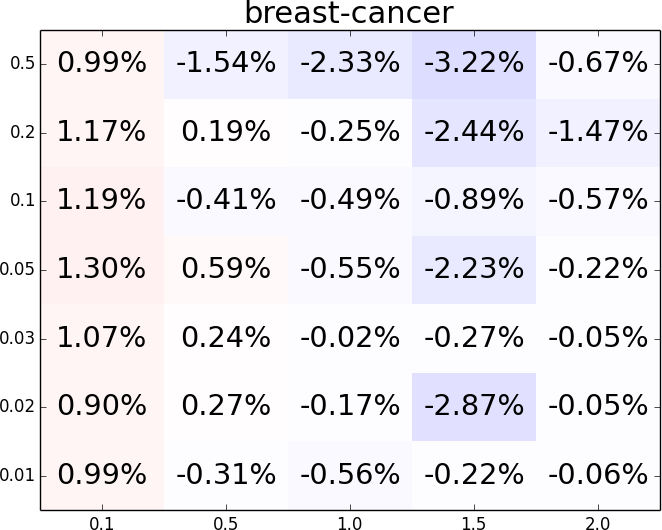}
 \includegraphics[width=0.19\textwidth]{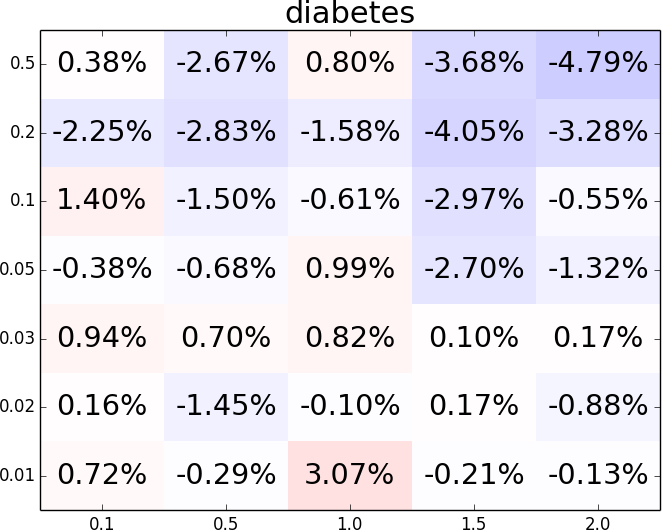}
 \includegraphics[width=0.19\textwidth]{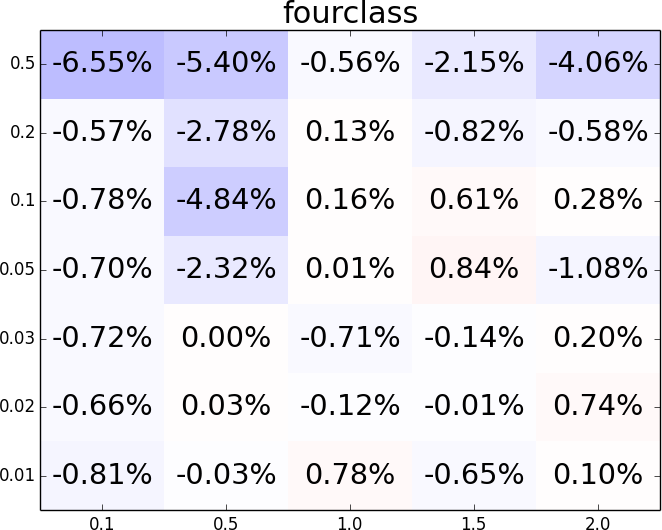}
 \includegraphics[width=0.19\textwidth]{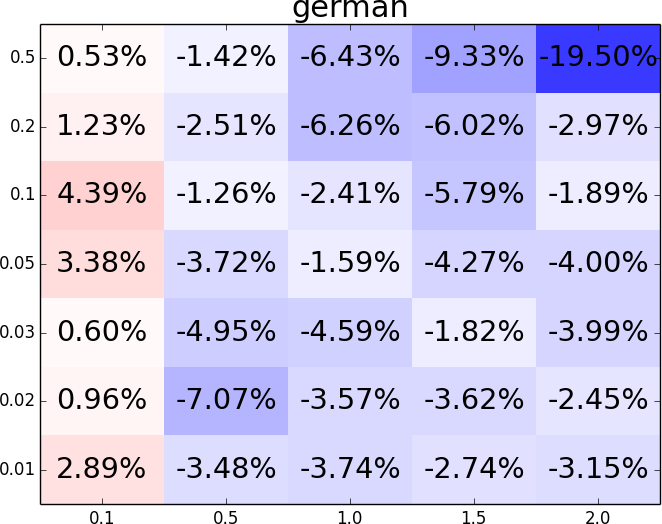}\\
 \includegraphics[width=0.19\textwidth]{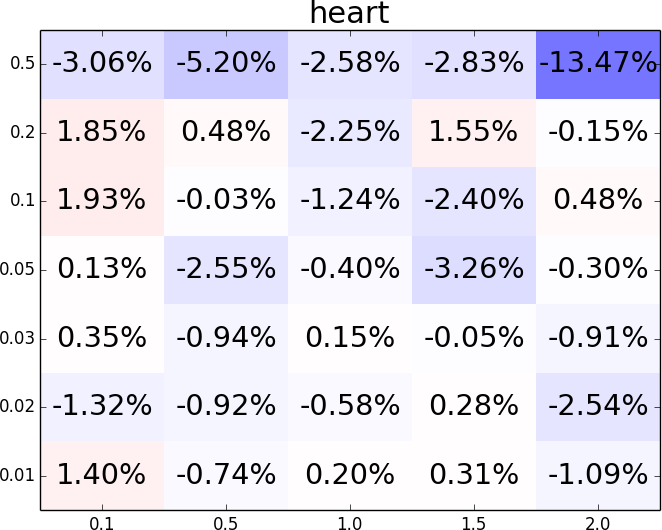}
 \includegraphics[width=0.19\textwidth]{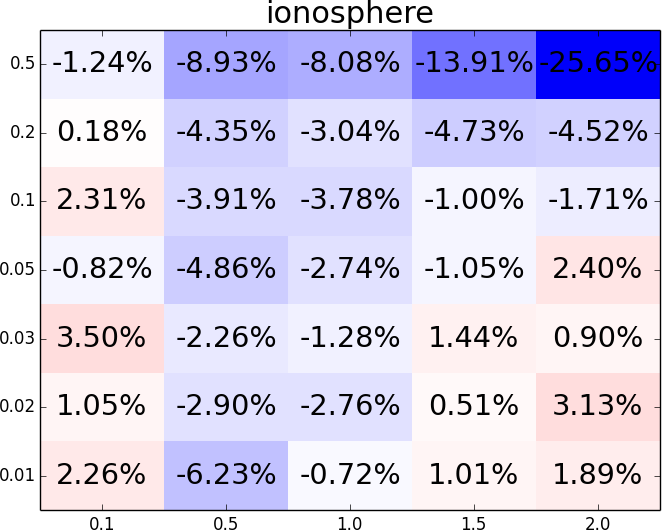}
 \includegraphics[width=0.19\textwidth]{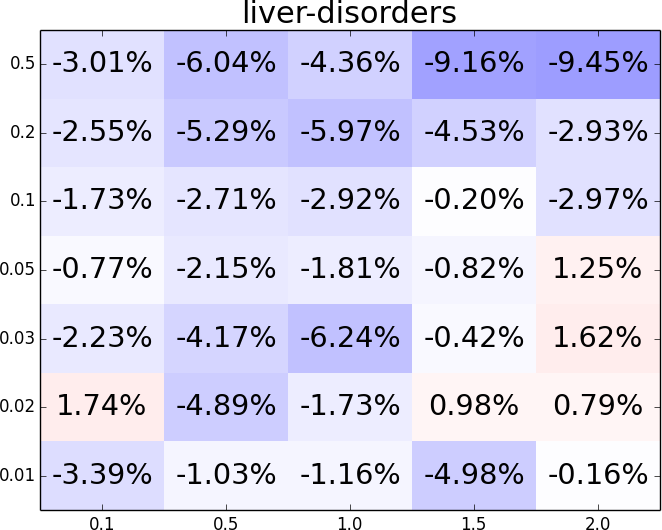}
 \includegraphics[width=0.19\textwidth]{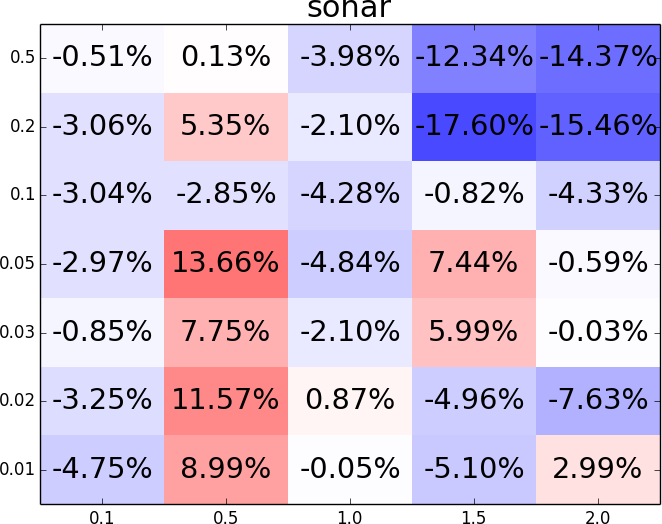}
 \includegraphics[width=0.19\textwidth]{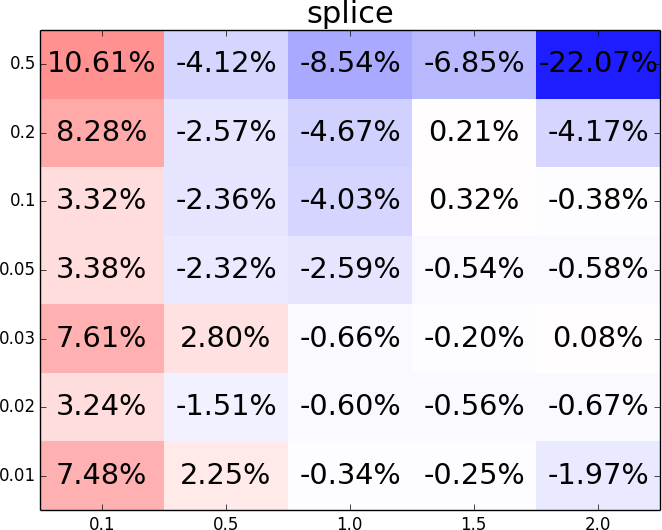}\\
 \vspace{0.5cm}
  \includegraphics[width=0.19\textwidth]{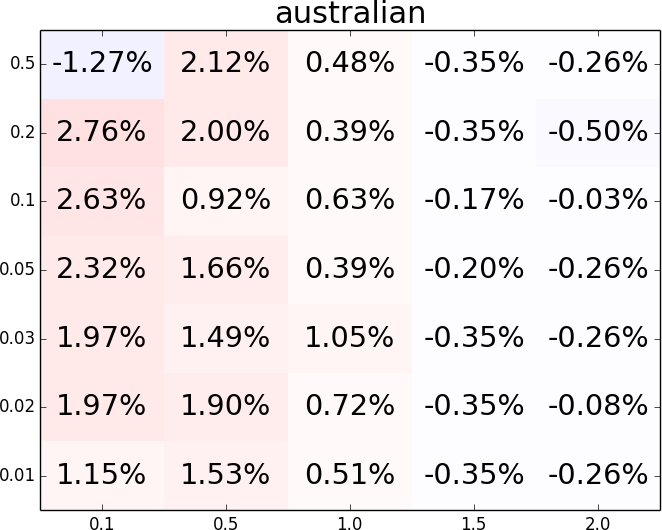}
 \includegraphics[width=0.19\textwidth]{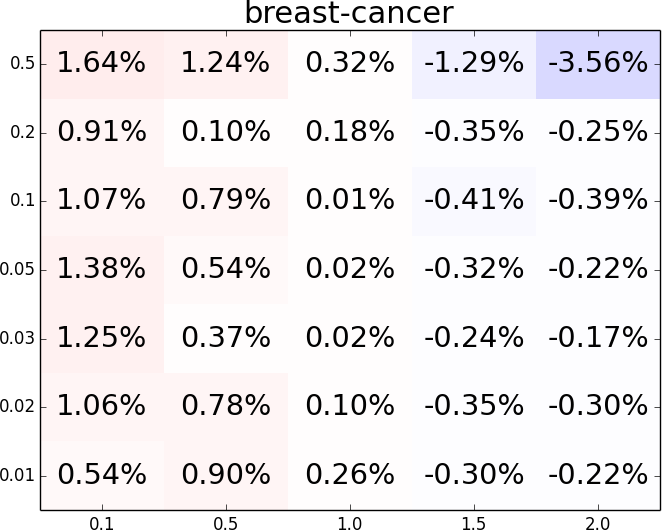}
 \includegraphics[width=0.19\textwidth]{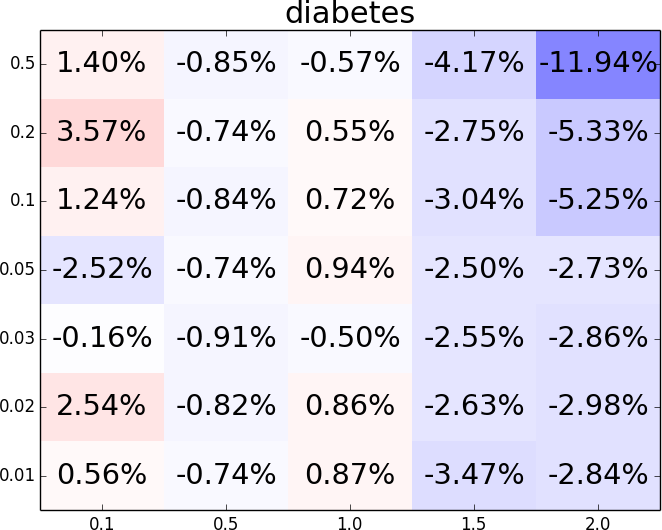}
 \includegraphics[width=0.19\textwidth]{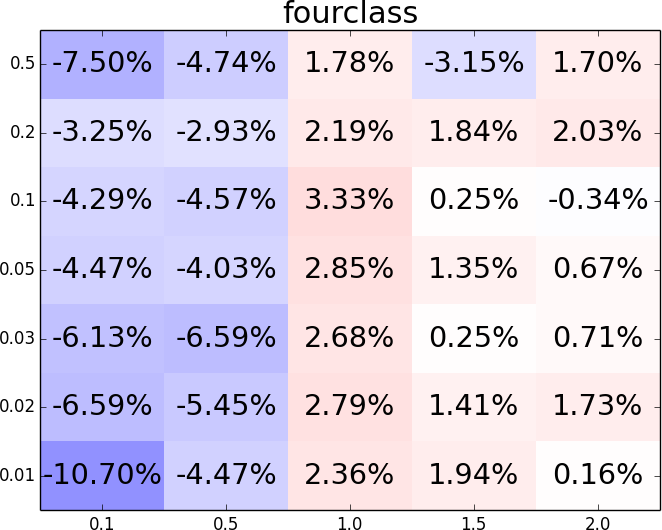}
 \includegraphics[width=0.19\textwidth]{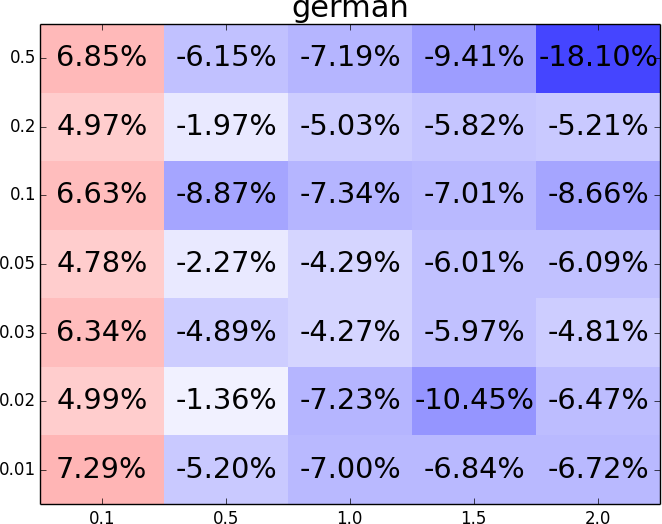}\\
 \includegraphics[width=0.19\textwidth]{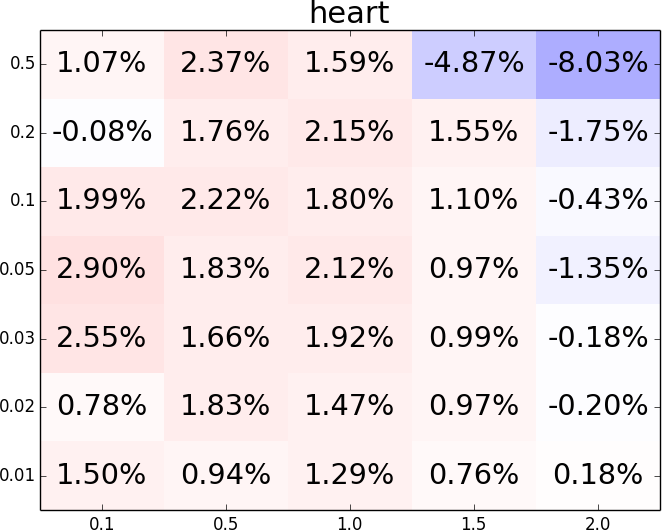}
 \includegraphics[width=0.19\textwidth]{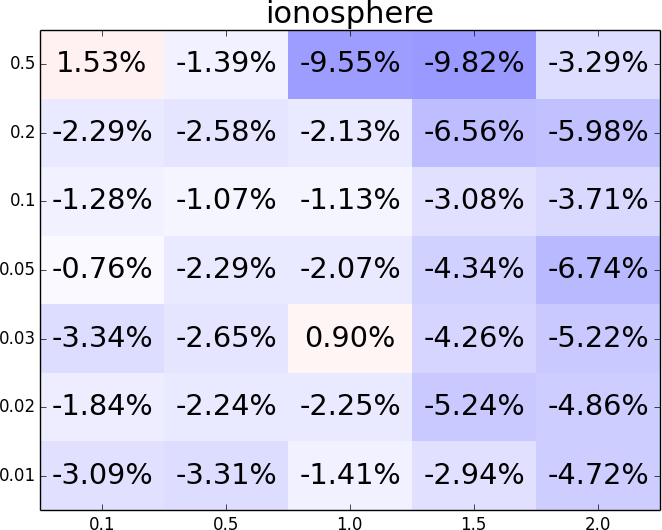}
 \includegraphics[width=0.19\textwidth]{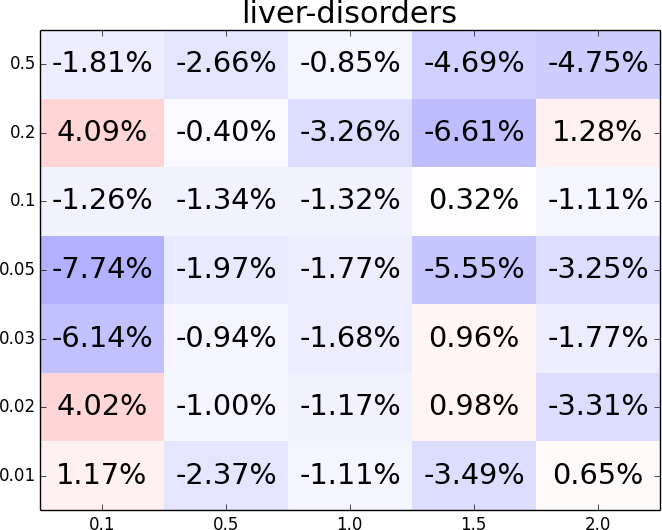}
 \includegraphics[width=0.19\textwidth]{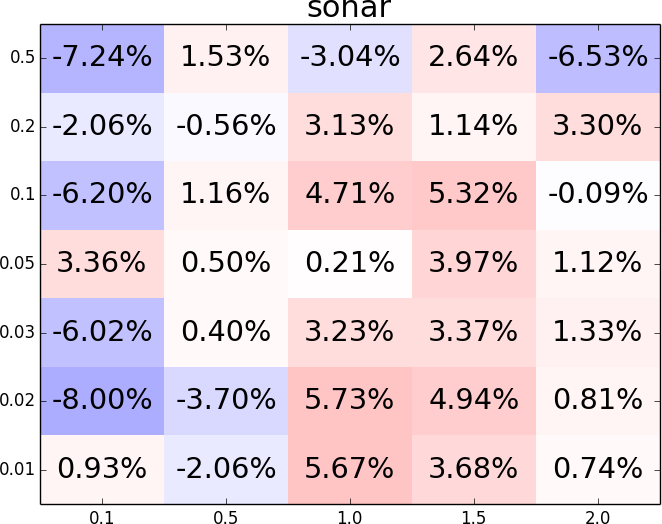}
 \includegraphics[width=0.19\textwidth]{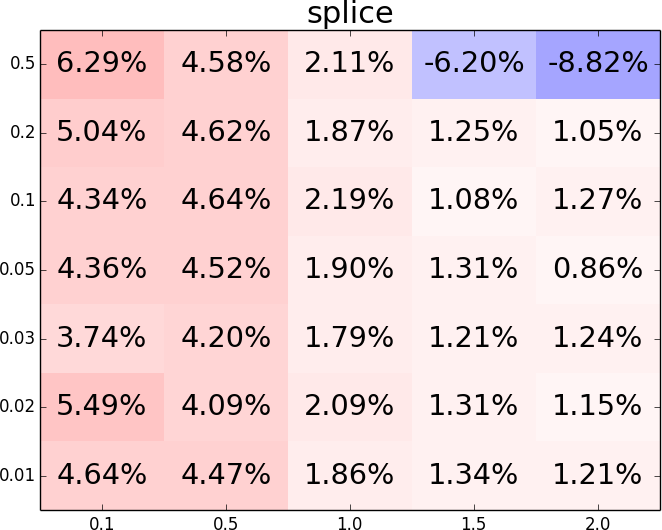}
\end{center}
 \caption{Comparison of the cross validation BAC scores between given approximated strategy (two top rows sorting and discarding, two bottom ones binning), $\gamma$ hyperparameter of $\dcsSymbol$ (x-axis), accepted error $\epsilon$ (y-axis). Positive values (and corresponding red colors) represent decrease in BAC score while negative values and corresponding blue colors -- increase after using approximated method.}
 \label{fig:comp}
\end{figure}
Even more interesting is the fact that for many experiments we actually noticed increase in the BAC score (bluish elements). This might be the consequence of more rough evaluation of the function (and gradient) values leading to optimization less prone to falling into local maxima. Our hypothesis is that it acts like a regularization helping to train MELC model.

% 
% \begin{tabular}{lrrrr}
% \toprule
% method 		       &   \multicolumn{2}{c}{CG}       & \multicolumn{2}{c}{L-BFGS-B}          \\
% name 		       &   bin &  dist &       bin &  dist \\
% \midrule
% australian      &  0.42 &  0.86 &      1.32 &  1.49 \\
% breast-cancer   &  0.59 &  0.98 &      0.50 &  0.70 \\
% diabetes        &  0.80 &  1.36 &      1.25 &  1.36 \\
% fourclass       &  0.75 &  1.07 &      0.99 &  1.38 \\
% german.numer    &  0.75 &  1.10 &      0.54 &  1.03 \\
% heart           &  0.51 &  0.96 &      1.13 &  1.10 \\
% ionosphere      &  0.08 &  0.45 &      0.36 &  0.71 \\
% liver-disorders &  0.58 &  0.78 &      1.28 &  1.22 \\
% sonar           &  0.13 &  0.49 &      0.66 &  1.20 \\
% splice          &  0.29 &  0.57 &      1.17 &  1.38 \\
% toy                    &  0.91 &  0.96 &      0.86 &  1.34 \\
% \bottomrule
% \end{tabular}

Analysis of the number of iterations of each optimization method required to converge (see Table~\ref{tab:convergence}) shows that both approximations significantly simplify the problem. It is important to notice that the number of iterations is not the number of $\dcsSymbol$ function evaluations (as both Conjugate Gradients and L-BFGS-B evaluate it multiple times in each iteration, especially during line searches). Consequently, number of iterations cannot be used as a measure of optimization speed but it says much about the complexity of the function being maximized.
\begin{table}[htb]
\begin{center}
\begin{tabular}{lrrrrrrr}
\toprule
method &   \multicolumn{3}{c}{CG} & &  \multicolumn{3}{c}{L-BFGS-B}       \\
name &  bin &  $\dcsSymbol$ &  dist &     &  bin &  $\dcsSymbol$ &  dist \\
\midrule
australian      &    4 &   36 &    22 &    &    11 &   39 &    37 \\
breast-cancer   &    4 &   35 &     8 &    &     6 &   39 &    14 \\
diabetes        &    3 &   30 &    20 &    &    18 &   36 &    29 \\
fourclass       &    4 &   12 &    10 &    &     6 &   15 &    14 \\
german.numer    &    7 &   60 &    32 &    &     7 &   58 &    38 \\
heart           &    3 &   40 &    19 &    &    12 &   34 &    20 \\
ionosphere      &    5 &  600 &   216 &    &    18 &  384 &   152 \\
liver-disorders &    4 &   30 &    22 &    &    22 &   43 &    30 \\
sonar           &    4 &  262 &   115 &    &    15 &  139 &   100 \\
splice          &    4 &   92 &    26 &    &    14 &   65 &    41 \\
% toy                   &    4 &    9 &     3 &     &    6 &    9 &     9 \\
\bottomrule
\end{tabular}
\caption{Number of optimization methods' iterations.}
\label{tab:convergence}
\end{center}
\end{table}
This seems to confirm our claim that approximation works similar to the regularization and thus it reduces small irregularities of the error surface due to the removal of small elements from the $\cipSymbol$ internal summation.

Experiments also showed importance on the regularization technique added to perform out of sphere optimization. During maximization of $\dcsSymbol$ in sonar and german datasets, norms of $v$ rapidly grew to over $1000$ if we turn off this modification and still use CG/L-BFGS-B. As a result the optimization problem became extremely hard and we needed tens of thousands $\dcsSymbol$ evaluation in order to converge. Adding regularizing term reduced the norm to nearly $1$ and number of required function calls by two orders of magnitude.

\section{Conclusions}

In this paper we proposed two simple approximation schemes for faster computation of MELC objective function and its gradient. We proved that in order to achieve constant error bound during optimization one needs a specific adaptive strategy for each of them and gave a simple, closed form equations for setting required parameters based on the user-specified acceptable level of error in the $\cipSymbol$ function value. We also showed how one can easily change the objective function in order to use wide range of existing optimizers while at the same time still work near the unit sphere which, as described in the MELC theory~\cite{melc}, is important from the numerical point of view.

During extensive evaluation we confirmed that such approach is valid in terms of reducing the mean number of $\exp$ calls by even an order of magnitude while not sacrificing the resulting classifiers accuracy. In fact the experiments suggest that proposed method acts like some kind of regularization which might not only simplify the optimization problem but also slightly increase the obtained results.

\bibliographystyle{plain}
\bibliography{biblio}
\end{document}